\documentclass[11pt]{article}

\usepackage[utf8]{inputenc} 
\usepackage[T1]{fontenc}    
\usepackage[colorlinks]{hyperref}       
\usepackage{geometry} 

\usepackage{fancyhdr} 
\usepackage{amsthm}
\usepackage{url}            
\usepackage{booktabs}       
\usepackage{amsfonts}       
\usepackage{nicefrac}       
\usepackage{microtype}      
\usepackage{xcolor}         
\usepackage{multirow}
\usepackage{natbib}

\usepackage{amsmath,amsfonts,bm}

\usepackage{amsthm}
\usepackage{thmtools}
\usepackage{thm-restate}

\usepackage{cleveref}
\usepackage{makecell}
\usepackage{diagbox}
\usepackage{tabu}
\usepackage{graphicx}
\graphicspath{ {./fig/} }
\usepackage{adjustbox}
\usepackage[inline]{enumitem}
\DeclareMathOperator*{\opodot}{\odot}

\newtheorem{assumption}{Assumption}[section]
\newtheorem{theorem}{Theorem}[section]

\newtheorem{corollary}{Corollary}[section]

\newtheorem{definition}{Definition}[section]
\newtheorem{remark}{Remark}[section]

\def\eqref#1{equation~\ref{#1}}

\def\1{\bm{1}}

\DeclareMathAlphabet{\mathsfit}{\encodingdefault}{\sfdefault}{m}{sl}
\SetMathAlphabet{\mathsfit}{bold}{\encodingdefault}{\sfdefault}{bx}{n}

\def\gD{{\mathcal{D}}}

\def\gF{{\mathcal{F}}}

\def\gH{{\mathcal{H}}}

\def\gO{{\mathcal{O}}}

\def\gT{{\mathcal{T}}}

\def\gW{{\mathcal{W}}}
\def\gX{{\mathcal{X}}}

\def\sP{{\mathbb{P}}}

\newcommand{\E}{\mathbb{E}}

\newcommand{\R}{\mathbb{R}}

\DeclareMathOperator*{\argmin}{arg\,min}

\newcommand{\p}[1]{\left( #1 \right)}
\newcommand{\ip}[2]{\left\langle #1, #2 \right \rangle}
\newcommand{\abs}[1]{\left | #1 \right |}
\newcommand{\norm}[1]{\left \| #1 \right \|}

\newcommand{\cO}{\mathcal{O}}
\newcommand{\fR}{\mathfrak{R}}
\newcommand{\fG}{\mathfrak{G}}
\newcommand{\wh}{\widehat}

\title{
{\bf{\LARGE{Provable Adaptation across Multiway Domains via Representation Learning}}}
}

\author{%
  Zhili Feng\thanks{Work done as an intern at NEC Labs America.} \\
  Carnegie Mellon University\\
  \texttt{zhilif@andrew.cmu.edu} \\
  \and 
  Shaobo Han\\
  NEC Laboratories America, Inc.\\
    \texttt{  shaobo@nec-labs.com} 
 \and 
 Simon S. Du\\
 University of Washington\\
 \texttt{ssdu@cs.washington.edu}\\
}

\begin{document}

\maketitle

\begin{abstract}

This paper studies zero-shot domain adaptation where
 each domain is indexed on a multi-dimensional array, and we only have data from a small subset of domains.
 Our goal is to produce predictors that perform well on \emph{unseen} domains.
We propose a model which consists of a domain-invariant latent representation layer and a domain-specific linear prediction layer with a low-rank tensor structure.
Theoretically, we present explicit sample complexity bounds to characterize the prediction error on unseen domains in terms of the number of domains with training data and the number of data per domain.
To our knowledge, this is the first finite-sample guarantee for zero-shot domain adaptation.
In addition, we provide experiments on two-way MNIST and four-way fiber sensing datasets to demonstrate the effectiveness of our proposed model.


\end{abstract}

\section{Introduction}
In many applications, collected datasets are cross-classified by a number of categorical factors such as  a set of experimental or environmental conditions that are either controllable or known. 
These applications can be formalized as a \emph{zero-shot domain adaptation} (ZSDA) problem with each domain being indexed on a multi-dimensional array.
Domain-shifts often exist among data collected under different conditions, and to minimize the bias, one would like to collect diverse data to cover all conditions. However, due to high cost of data collection or insufficient resource in the field, it is often impractical to cover all possible combinations of factors. In most cases, training data collection can only cover a few limited, accessible scenarios, and the data from target domain may not be available during the training phase at all. 

One motivating example is the distributed fiber optical sensing \citep{rogers1988distributed} technique started to be used in smart city applications  \citep{huang2019first}. Collecting fiber sensing data for machine learning is expensive and the characteristics of sensing data vary considerably across various heterogeneous factors, such as weather-ground conditions, signal source types, and sensing distances. Although conducting  experiments to cover a small subset of combinations of different factors is possible, it is difficult to conduct experiments to cover all possible combinations.
Specifically, in order to account for the effects of multiple factors (e.g., soil type $\times$ distance to sensor $\times$ weather), one has to visit multiple sites and set up the sensing equipment, and wait for the right time until the target type of data can be collected, and some combination of factors are not accessible due to the constraint of the physical world (e.g. a testbed at $10$km sensing distance with asphalt pavement surface).
 Besides, in the digital pathology application, domain shift can be caused by different clinical centers, scanners, and tumor types \citep{fagerblom2021combatting}. It is impossible to cover every combination in training (e.g., one center maybe specialized in certain tumor class or own one particular type of scanners).
Ideally, for these applications, one would want to train a model under a subset of scenarios, and adapt it to new \emph{unseen scenarios} without collecting new data.

Another motivating application to design data augmentation with composition of operations. 
For instance, in image recognition applications, there are often multiple operations can be combined into augmented datasets, including blurring, translating, viewpoint, flipping, scaling, etc, and each with several levels or values to choose from. It may require huge amount of augmented data to cover all  combinations of operations. Ideally, one only wants to try a small subset of all possible combinations.
Previous works have shown that composing operations in careful ways outperforms randomly using data augmentation schemes \citep{ratner2017learning,cubuk2019autoaugment}. 


%

%
In the general ZSDA formulation, each domain is associated with a task descriptor, which is used to generate a predictor for the domain.
This function that maps a task descriptor to a predictor, i.e., a predictor generator, is trained on training domains where data is available \citep{yang2014unified}. In our setting, we can view the categorical indices for each domain as a task descriptor and thus apply existing approaches.
However, this ignores the inherent multiway structure that connects different domains.  
Furthermore, existing approaches lack theoretical guarantees, and thus it is unclear under what conditions, the learned predictor generator can \emph{provably} output an accurate predictor on unseen domains.
A natural question is \begin{center}
\textbf{Can we design a \emph{provably sample-efficient} algorithm for zero-shot domain adaptation that exploits the multiway structure?}
\end{center}

\subsection{Our Contributions}
In this paper, we answer this question affirmatively. Our contributions are summarized below.
\begin{itemize}[leftmargin=*]
\item  we consider a multilinear structure which naturally characterize the relatedness among different domains.
Formally, we assume that there are totally $D=d^M$ domains that form a $d^{\times M}$-dimensional task tensor, where the predictor on domain $t\in[D]$ is parameterized by a tensor $w_t^*\circ\phi^*\in\gW\circ\Phi$ where $\phi^* \in \mathbb{R}^p$ is a common representation for tasks and $w_t \in \mathbb{R}^p$ is a domain specific linear predictor.
Note the linear predictors, $w^*_1,\ldots, w_D^*\in\R^p$, together form a form a $\underbrace{d\times d\times\cdots\times d}_{M~\text{times}}\times p$-dimensional tensor, which we denote as $\gT$. 
To enable domain adaptation, we further impose a low-rank multi-linear on the linear predictor $\gT$.
The structure we impose on the model is that  the slice along the last axis is always rank-$K$ with $K \ll d^M$. 
That is, for all $j\in[p]$, $\gT_{\cdot, j}\in\R^{d\times d\times\cdots\times d}$ has rank $K$. 
\item We provide a finite-sample statistics analysis for this model.
We show that if during training $T\in[D]$ source domains are observed, and in each observed domain $n$ samples are collected, then the expected excess risk across all $D$ domains is of order $\tilde\gO\p{\p{\frac{TC(\gW)+C(\Phi)}{nT}}^{1/4}+p\p{\frac{KdM^2}{T}}^{1/2} }$, where $C(\gF)$ represents the complexity of function class $\gF$.
The first term corresponds to the cost of learning the common representation and domain-specific linear classifiers, and the second term corresponds to the cost of generalizing to unseen domains. 
In the case where $n$ is large, we can simply ignore the first term, and conclude that $T$ should scale with $Kd(pM)^2$. 
This theoretically justifies our method can adapt to \emph{unseen} domains with limited training domains.
To our knowledge, this is the first finite-sample guarantee for ZSDA.

\item We test our proposed method on a two-way MNIST dataset and four-way fiber sensing datasets.
The empirical results show our method matches the performance of general ZSDA algorithm by \cite{yang2014unified} which has a lot more parameters than our method.

\end{itemize}

\section{Related Work}\label{sec:related}
For theoretical results of multi-task learning, \citet{baxter2000model} first derived a covering number bound for sample complexities. Later \citet{maurer2016benefit} gave a uniform convergence bound for multitask representation learning. Their bound is of order $\gO(\sqrt{{1}/{m}} + \sqrt{{1}/{T}})$, where $m$ is the number of training samples in the target domain, so it is not reflecting how their learning process benefits from seeing multiple domains. Recently, \citet{du2020few,tripuraneni2020theory} successfully derived results of type $\gO(\sqrt{{1}/{m}} + \sqrt{{1}/{(nT)}})$, which satisfactorily explained how learning algorithm benefits from seeing abundant training data in many source domains, even if training data in target domain is scarce. This setting is often referred to as transfer learning or few-shot learning. A key difference between this setting and ours is that we do not see any training data from the target domain.

Many attempts to solving multi-task learning via low-rank matrices/tensors have been proposed \cite{romera2013multilinear, signoretto2013learning, wimalawarne2014multitask}, focusing on regularizing shallow models. \citet{wimalawarne2014multitask} specifically discussed task imputation and gave a sample complexity bound. However, their analysis assumed each training data is uniformly sampled from all source domains, so it is incomparable to our result. \citet{yang2016deep,li2017deeper} used low-rank tensor decomposition to share parameters between neural architectures for multiple tasks and empirically verified that both multi-task learning and domain generalization can benefit from the low-rank structure. Some different but closely related settings include unsupervised domain adaptation and zero-shot learning, where the test labels are either unseen during the training phase or different from what the model has been trained with. This is beyond the scope of this paper, as we only consider the case where all domains share same set of labels. 

%
%

The low-rank matrix completion problem has been thoroughly studied in the past decade \cite{srebro2005rank,NIPS2004_f1b07759,candes2009exact,recht2011simpler}. Despite the combinatorial nature of rank, it has a nice convex relaxation, the trace norm, that can be optimized efficiently, in both noiseless and noisy settings \cite{wainwright2019high}. Although these quantities (trace norms, ranks, etc) are easy to calculate for matrices, most quantities associated with tensors are NP-hard to compute, including the trace norm \cite{hillar2013most}. There are many efforts in the theoretical computer science community to tackle tensor completion problems from the computational complexity perspective \cite{barak2016noisy,liu2020tensor}. In this paper, we use the statistical properties of tensor completion despite their computational inefficiency. Hence our sample complexity bound does not contradict the seemingly worse bounds in the theory literature.

\section{Preliminary and Overview}\label{sec:prelim}

\paragraph{Notation.} For $D\in\mathbb Z$, we let $[D]=\{1,2,\ldots, D\}$. We write $\|\cdot\|$ to represent the $\ell_2$ norm or Frobenius norm. Let $\ip{\cdot}{\cdot}$ be the inner product on an inner product space. Throughout the paper, we assume for simplicity in total there are $D=\underbrace{d\times d\times\cdots\times d}_\text{$M$}$ data domains. One can easily generalize the results to the general case where $D=\prod_{i=1}^M d_i$. We use $\otimes$ to denote the tensor product and $\odot$ the Hadamard product. We use $T\in[D]$ to denote the number of seen source domains, $n$ to be the amount of sample in a single source domain. Since each $t\in[D]$ also uniquely corresponds to a multi-index in $I\in [d]^{\times M}$, if $\gT\in \R^{d\times\cdots \times d\times p}$ is a $d^{\times M}$ dimension tensor composing of $p$-dimensional linear functionals, we also write $\gT_{t,\cdot}\in\R^p$ to represent the classifier at multi-index $I$, and we use $t$ and $I$ as indices interchangeably. 

\subsection{Problem Setup}
During training, we have $T$ source domains out of the $D$ domains uniformly at random. For each domain $t\in[T]$, we have $n$ i.i.d data $\{(x_{t,i}, y_{t,i})_{i=1}^n\}$ samples from the following probability~\citep{tripuraneni2020theory}:
$
	\sP_t(x,y)=\sP_{w_t^*\circ\phi^*}(x,y)=\sP_x(x)\sP_{y|x}(y|w_t^*\circ\phi^*(x)),
$
where all domains share the common feature mapping $\phi^*:(\R^r\to\R^p) \cap \Phi$ and a common marginal distribution $\sP_x$. Each domain has its specific classifier $w_t^*\in(\R^p\to\R)\cap\gW$. $\Phi$ can be a general function class and $\gW$ composes of Lipschitz linear functionals. We denote this distribution as $\gD_t$.

The key assumption is that $D$ domain-specific classifiers form a $D\times p$ rank-$K$ tensor $\gT$ in the following sense:
\begin{align}\label{eq:lowrankform}
	\forall I\in [d]^{\times M}: \gT_{I,\cdot}^* = w_t^* = \sum_{k=1}^K\opodot_{m=1}^M \alpha_{k, t_m}^*,
\end{align} 
for $\alpha_{k, t_m}^*\in \R^p$. We remark that this does not mean $\gT$ has rank $K$. Instead, this means for each $j\in[p]$, $\gT_{\cdot, j}$ has rank $K$.

Let $\ell:\R\times \R\to\R_{\geq 0}$ be the loss function, $L((w\circ\phi)(x), y)=\E_{x, y}[\ell((w\circ\phi)(x), y)]$ be the expected loss, $\wh L_n((w\circ\phi)(x), y)=\frac{1}{n}\sum_{i=1}^n\ell((w\circ\phi)(x_i), y_i)$ be the empirical loss. When $x$ and $y$ are clear from the context, we write $L(w\circ\phi)$ to denote $L((w\circ\phi)(x), y)$, and similarly for $\ell$ and $\wh L_n$. When the number of samples $n$ is clear from the context, we write $\wh L:= \wh L_n$. We write $L_t:=\E_{\gD_t}[\ell(\cdot,\cdot))]$ to emphasize the data is from $\gD_t$. Our goal is to minimize the \emph{average excess risk} across all domains:
$
	\frac{1}{D}\sum_{t\in[D]}\E_{\gD_t}\left[ \ell\p{\hat w_t\circ\hat\phi} -\ell\p{ w_t^*\circ\phi^*} \right].
$

Our learning algorithm first outputs the empirical risk minimizer (ERM) on the seen source domains:
\begin{align}\label{eq:erm}
\begin{split}
	&\hat w_t, \hat\phi= \argmin_{\substack{\phi\in\Phi,\\ w_1,\ldots, w_t\in\gW}}\frac{1}{nT}\sum_{t=1}^T\sum_{i=1}^n \ell\p{(w_t\circ\phi)(x_{t,i}), y_{t,i}}.
\end{split}
\end{align}

Naturally we put $\hat w_t$ into a $d^{\times M}\times p$ tensor $\tilde\gT$, such that $\tilde\gT_{t,\cdot}=\hat w_t$. Then we do a tensor completion for each $j\in[p]$ separately:
\begin{align}\label{eq:tensorcompletion}
	\wh\gT_{\cdot, i} = \argmin_{\text{rank-}K\ \gT} \frac{1}{T}\sum_{t\in[T]}\abs{\gT_{t, i}-\tilde\gT_{t, i}}.
\end{align}
Finally, at test time, for target domain indexed by $t\in[D]$ and test data $x\sim \gD_t$, we make the prediction $\hat y = (\hat w_t \circ \hat\phi)(x)$, where $\hat w_t = \wh\gT_{t, \cdot}$

\paragraph{Notion of Complexity} We use Gaussian complexity and Rademacher complexity to measure the function class complexity. 
\begin{definition}[Gaussian and Rademacher complexity]\label{def:complexity}
Let $\gF\subset \{f: \R^p\to\R^q\}$ be a vector-valued function class, the empirical Gaussian and Rademacher complexity on $n$ points $X=\{x_1,\ldots, x_n\}$ are defined as:
\begin{align*}
	&\wh \fG_n(\gF_{| X})=\E_{g}\left[\sup_{f\in\gF} \frac{1}{N}\sum_{i\in[N]}  \sum_{k\in[q]}g_{ki}f_k(x_i)\right], \quad \wh \fR_n(\gF_{| X})=\E_{\epsilon}\left[\sup_{f\in\gF} \frac{1}{N}\sum_{i\in[N]}\sum_{k\in[q]}\epsilon_{ki}f_k(x_i)\right], \\
\end{align*}
where $g_{ki}$'s are all standard Gaussian random variables, $\epsilon_{ki}$'s are Rademacher random variables, and $f_k$ is the $k$th entry of $f$. 
\end{definition}
Let $\fG_n(\gF)=\E_X[\wh \fG_n(\gF_{| X})]$ denote the expected Gaussian complexity (similarly for Rademacher complexity). To deal with the divergence among domains, we adopt the following notions from \citet{tripuraneni2020theory}: $\overline{\mathfrak{G}}_{n}(\mathcal{W})=\max _{Z\in \mathcal{Z}} \hat{\mathfrak{G}}_n(\mathcal{W}_{|Z})$, where $\mathcal{Z}=\left\{\left(\phi\left(x_{1}\right), \cdots, \phi\left(x_{n}\right)\right) \mid \phi \in \Phi, x_i \in \mathcal{X} \text { for all } i \in[n]\right\}$.

To measure the complexity of tensors, we use the notion of pseudo-dimension. 
\begin{definition}[Pseudo-dimension]
	Let $\gF\subseteq \R^\gX$ be a real-valued function class. Let $x_1^m=(x_1,\ldots, x_m)\in\gX^m$. We say $x_1^m$ is pseudo-shattered by $\gF$ is there exists $r=(r_1,\ldots, r_m)\in\R^m$ such that for all $b=(b_1,\ldots, b_m)\in\{\pm 1\}^m$, there exists $f\in\gF$ such that $\operatorname{sign}(f_b(x_i)-r_i)=b_i$ for all $i\in[m]$. The pseudo-dimension of $\gF$ is defined as:
	\begin{align*}
		\operatorname{Pdim}(\gF)
		=\max\bigg\{&m\in\mathbb N: \exists x_1^m \text{s.t } x_1^m \text{ is pseudo-shattered by }\gF\bigg\}.
	\end{align*}
\end{definition}

\section{Theoretical Results}\label{sec:theory}
In this section, we derive the sample complexity bound for our DG setting. As a high level overview, we first prove that on the $T$ seen source domains, with enough training data and sufficient regularity conditions, we can have $\hat w_t$ being sufficient close to $w_t^*$ on the source domains. Then we show that with a large enough $T$, even on an unseen domain $t\in[D]\backslash[T]$ we can also approximate $w_t^*$ well by completing the low-rank tensor formed by learned $\hat w_t$. Finally, by certain regularity conditions on the loss function, we can show the excess risk is also small.

We require the following regularity assumptions for our theoretical development.

\begin{assumption}\label{assum}
	\begin{enumerate}
		\item 
			The learning problem is realizable, that is, $w_t^*\in \gW$ for all $t\in[D]$ and $\phi^*\in\Phi$. WLOG assume that $w^*_t\circ\phi^*$ is the unique minimizer of $L_t$ for all $t$.
		\item 
			$\ell(\cdot,\cdot)$ is $B$-bounded, and $\ell(\cdot, y)$ is $L$-Lipschitz.
		\item	
			For all $w\in\gW$, $\|w\|\leq W$.
		\item 
			$\sup_{x}\|\phi(x)\|\leq D_\mathcal{X}$, for any $\phi\in\Phi$.
		\item 
			For all $t\in[D]$, 
			$L_t$ is $\lambda$-strongly convex in $w_t$ for $\phi^*$.

	\end{enumerate}
\end{assumption}
The first assumption is indeed general. Since we care about the excess risk, in realizable problems we can compare to the unique minimizer of $L$ in $\gW\circ\Phi$ instead of $w_t^*\circ\phi^*$. The existence of a unique minimizer is necessary, otherwise tensor completion can lead to arbitrary errors. Assumptions 2-4 are common. Lastly, without strong convexity in the last assumption, one cannot derive the closeness of $\hat w_t$ to $w_t^*$ from the closeness of the excess risk on source domains.

\subsection{Learning Common Representation}
In this subsection, we discuss how to learn the shared representation. The proof follows \citet{tripuraneni2020theory} with modifications to adopt to our setting. First we introduce some definitions that defines the divergence among the source domains and target domains.

\begin{definition}\label{def:diverse}
	For a function class $\gW$, $T$ functions $w_1,\ldots, w_T$, and data $(x_t, y_t)\sim \gD_t$ for any $t\in[T]$, the \textbf{task-average representation difference (TARD)} between representation $\phi$ and $\phi'$ is defined as:
	$$
		\bar d_{\gW, w}(\phi'; \phi)=\frac{1}{T}\sum_{t\in[T]}\inf_{ w'\in \gW}\mathop{\E}_{\gD_t}\p{\ell( w'\circ\phi')-\ell( w_t\circ\phi)}.
$$

Let $w^*,\phi^*$ be the true underlying functions. Define the \textbf{uniform absolute representation difference (UARD)} to be
	$$
		 d_{\gW}(\phi'; \phi)=\sup_{t\in[D]} \sup_{w_t\in\gW}\abs{\mathop{\E}_{\gD_t}\p{\ell( w_t\circ\phi')-\ell( w_t\circ\phi)}}.
	$$
	
For $\gW$, we say  $w=\{w_1,\ldots, w_T\}$	is \textbf{$(\boldsymbol\nu,\boldsymbol\epsilon)$-diverse} for a representation $\phi$, if the following holds for all $\phi'\in \Phi$:
$$
	 d_{\gW}(\phi'; \phi)\leq \bar d_{\gW, w}(\phi'; \phi)/\nu+\epsilon.
$$\end{definition}

The notion of task-average representation difference was introduced by \citet{tripuraneni2020theory}. In their setting, they bound the \textsl{worse-case representation difference (WCRD)} between two representations $\phi, \phi'\in\Phi$ which is defined as
$$
	\sup_{w\in\gW}\inf_{w'\in\gW}\mathop{\E}_{(x, y)\sim\mathbb{P}_{w\circ\phi}}[\ell(w'\circ\phi')-\ell(w\circ\phi)],
$$
using the task-average representation difference. One should note that WCRD is changing the data distribution, while UARD takes expectation over the true distribution $\gD_t$. We are saying that under the true distribution, the worst case difference between using $\phi$ vs. $\phi'$, over both the choice of linear classifiers $w$ and the domains, is not much larger than TARD. Intuitively, this says that our source domains and task domains are benign enough, such that the ERM $\hat\phi$ performs similarly to the true $\phi^*$.

Although it seems that $\bar d_{\gW, w}(\phi'; \phi)$ can be negative and $d_{\gW}(\phi'; \phi)$ is always positive, later we only consider $\bar d_{\gW, w^*}(\phi'; \phi^*)$ and $w^*_t\circ\phi^*$ is assumed to be the minimizer. Hence, $\bar d_{\gW, w^*}(\phi'; \phi^*)$ is always positive, and our notion of task diversity makes sense. 

We cite the following theorem from \citet{tripuraneni2020theory}.
\begin{theorem}\label{thm:repdiv}
	Let $\hat\phi$ be the ERM in \cref{eq:erm}. Under \cref{assum}, with probability at least $1-\delta$, we have
	\begin{align*}
	\begin{split}
			&\bar d_{\gW,w^*}(\hat\phi, \phi^*)\leq 4096 L\left[\frac{WD_{\mathcal{X}}}{(n T)^{2}}+\log (n t) \cdot\left[W \cdot \mathfrak{G}_{n T}(\Phi)+\overline{\mathfrak{G}}_{n}(\mathcal{W})\right]\right]+8 B \sqrt{\frac{\log (2 / \delta)}{n T}}.
	\end{split}
	\end{align*}
\end{theorem}

\subsection{Learning Domain Specific Linear Layer and Tensor Completion}
Now we give results for learning the domain specific linear classification. 

\begin{restatable}{theorem}{lineardiff}
\label{thm:lineardiff}
	Let $\hat w, \hat\phi$ be the ERM in  \cref{eq:erm}. Let \cref{assum} holds. With probability at least $1-\delta$, we have

\begin{align*}
\resizebox{1\hsize}{!}{%
		$\frac{1}{T}\sum_{t\in[T]}\|\hat w_t-w^*_t\|_2 \leq \gO\Bigg(\sqrt{\frac{2}{\lambda}}\Bigg( L\Big[\frac{WD_{\mathcal{X}}}{(n T)^{2}}+\log (n T)\cdot\left[W \cdot \mathfrak{G}_{n T}(\Phi)+\overline{\mathfrak{G}}_{n}(\mathcal{W})\right]\Big] +8 B \sqrt{\frac{\log (2 / \delta)}{n T}}\Bigg)^{1/2}\Bigg).$
	}
\end{align*}
\end{restatable}

The proof of \Cref{thm:lineardiff} starts with showing that $L(\hat w_t\circ \phi^*)$ is close to $L(w_t^*\circ\phi^*)$, while a normal ERM analysis only asserts $L(\hat w_t\circ \hat\phi)$ is close to $L(w_t^*\circ\phi^*)$. Such assertion holds by the $(\nu,\epsilon)$-diverse assumption. Intuitively, such diversity assumption makes sure our learning landscape is somehow smooth: $\hat w_t$ should be stable such that if we perturb $\hat\phi$ to $\phi^*$, the loss function does not alter too much. Together with strong convexity, these conditions guarantee $\hat w_t$ is close to the true $w_t^*$ on average.

After we learn $\hat w_t$ for all $t\in[T]$, we find another ERM $\wh\gT$ by constructing $p$ tensors $\wh\gT_{\cdot, i}$ for all $i\in[p]$ separately as in \cref{eq:tensorcompletion}.
%
Since the entries are observed uniformly at random, taking the expectation over this uniform distribution, we have
\begin{align}\label{eq:expecttensor}
	\E\left[ \frac{1}{T}\sum_{t=1}^T |\wh\gT_{t,i}- \tilde\gT_{t, i}| \right] = \frac{1}{D}\sum_{t=1}^D |\wh\gT_{t,i}-\tilde\gT_{t, i}|.
\end{align}
We only observe the empirical tensor completion loss $\frac{1}{T}\sum_{t=1}^T |\wh\gT_{t,i}- \tilde\gT_{t, i}|$ during training. With large enough $T$, we want it to converge to the expected tensor completion loss \cref{eq:expecttensor}. This section is meant to answer how large $T$ should be.

\begin{remark}\label{remark:tensorcomplete}
	Recall that $\tilde\gT$ is the tensor formed by setting $\tilde \gT_{t, \cdot}=\hat w_t$ for $t\in[T]$. Technically, for $t\in[D]\backslash[T]$, $\tilde \gT_{t, \cdot}$ is undefined. One should imagine for those $t$, we hypothetically also learn $\hat w_t$ using \cref{eq:erm} as well. Hence, for any set of $T$ source domains, $\frac{1}{T}\sum_{t\in[T]}\tilde\gT_{t,\cdot}$ is $\epsilon(n)$ close to $\frac{1}{T}\sum_{t\in[T]}\gT^*_{t,\cdot}$.  This aligns with our learning process. Although in reality we never actually form the whole $\tilde\gT$, it is used only for analysis. Equivalently, one can interpret this as tensor completion with noisy observation $\tilde \gT$ from the true tensor $\gT^*$, where the noise is arbitrarily small depending on $n$.
\end{remark}

\begin{restatable}{lemma}{pseudodim}
\label{lma:pseudodim}
	Let $\gX_K$ be the class of rank-$K$ tensor of shape $d^{\times M}$, its pseudo-dimension can be bounded by $
		\operatorname{Pdim}(\gX_K)\leq KdM^2\log(8ed).$
\end{restatable}
$\operatorname{Pdim}(\gX_K)$ is computed by treating tensors as polynomials and counting the connected components of these polynomials. Even though any $X\in\gX_K$ has $d^M$ entries, its complexity only scales as $\operatorname{poly}(K, d, M)$. This assures that we only need to see polynomially many source tasks to perform well in unseen domains.
Using the pseudo-dimension, we have the following uniform convergence result.

\begin{restatable}{theorem}{tensorcompletion}
\label{thm:tensorcompletion}
	For any fixed $j\in[p]$, with probability at least $1-\delta$, 
		\begin{align}\label{eq:tensorcompleteunion}
	\begin{split}
			&\frac{1}{D}\sum_{t\in[D]} \norm{\wh\gT_{t,\cdot}- \gT^*_{t,\cdot}}
			\leq \frac{1}{T}\sum_{t\in[T]} \sum_{j\in[p]}\abs{\wh\gT_{t,j}-\tilde\gT_{t,j}}+p\sqrt{\frac{KdM^2\log\p{8ed}+\log(p/\delta)}{T}}+\tilde\cO(n^{-1/4}).
	\end{split}
	\end{align}
\end{restatable}

The last $\cO(n^{-1/4})$ term in \cref{eq:tensorcompleteunion} comes from \Cref{thm:lineardiff}. This is the cost we pay for not knowing the true $\gT_{t, \cdot}^*$. If in each source domains we have infinity number of training samples, then statistically \Cref{thm:lineardiff} recovers the true $\gT_{t, \cdot}^*$. In this case, we only need to observe $T=\operatorname{poly}(p,K,d,M)$ source domains.

\subsection{Main Theorem}
We are now ready to show our main result. 
\begin{restatable}{theorem}{mainthm}
\label{thm:mainthm}
	Let \cref{assum} holds and $w^*=\{w_1^*\ldots, w_T^*\}$ being $(\nu,\epsilon)$-diverse for representation $\phi^*$. Let $\bar d_{\gW, w^*}(\hat\phi, \phi^*)$ be as defined in \cref{def:diverse}. With probability at least $1-3\delta$, the following holds:
	\begin{align*}
		&\frac{1}{D}\sum_{t=1}^D\mathop E_{(x,y)\sim \gD_t}[\ell(\hat w_t\circ\hat\phi(x), y)-\ell( w^*_t\circ\phi^*(x), y)]\leq\frac{LD_{\gX}W}{T}\sum_{t\in[T]} \sum_{j\in[p]}\abs{\hat w_{t,j}-\tilde\gT_{t,j}}\\
		&\quad+LD_{\gX}Wp\sqrt{\frac{KdM^2\log\p{8ed}+\log(p/\delta)}{T}} + \tilde\gO\Bigg(\frac{C(\gW)}{n}+\frac{C(\Phi)}{nT}\Bigg)^{1/4}.
	\end{align*}
\end{restatable}
The first two lines correspond to the cost of tensor completion and the last term corresponds to predicting with inaccurate $\hat\phi$. The term $\sum_{t\in[T]} \sum_{j\in[p]}\abs{\hat w_{t,j}-\wh\gT_{t,j}}$ can be exactly computed. Recall $\hat w_t$ is the learned by \cref{eq:erm} and $\wh\gT_{t,\cdot}$ is the estimated linear classifier after tensor completion. As $n$ increases, this difference becomes smaller as $\hat w_t$ becoming close to $w^*_t$.  

\section{Experiments}\label{sec:experiment}

We empirically verify that our proposed model leads to better generalization performance than baseline model with vanilla representation learning on both a variant of MNIST dataset and a fiber sensing dataset. On both datasets, we use LeNet trained on all available training data as the baseline. According to \cite{li2017deeper, li2019feature}, this is a simple yet strong baseline that outperforms most DG methods. Our model is trained in an end-to-end fashion. Instead of finding ERMs in \cref{eq:erm} and perform tensor completion, we directly represent $\hat w_t= \sum_{k=1}^K\opodot_{m=1}^M \hat\alpha_{k, t_m}$, and output
\begin{align}\label{eq:endtoendlowrank}
\begin{split}
	&\hat \alpha_{k, t_m}, \hat\phi
	= \argmin_{\substack{\phi\in\Phi,\\\alpha_{k, t_m}}}\frac{1}{nT}\sum_{t=1}^T\sum_{i=1}^n \ell\p{( \sum_{k=1}^K\opodot_{m=1}^M\alpha_{k, t_m}\circ\phi)(x_{t,i}), y_{t,i}}.
\end{split}
\end{align}

In this way, we can fully exploit the computational convenience of auto-differentiation rather than dealing with the algorithmic difficulty of tensor completion. All models are trained using the cross entropy loss. To prevent overfitting, we stop training of all models on the two-way MNIST dataset as soon as the last $50$ iterations have average loss less than $0.05$, and the training of all models on a four-way fiber sensing dataset is stopped once the last $100$ iterations have average loss less than $0.05$. Throughout the experiments, the Adam optimizer with default learning rate $10^{-3}$ is used. The sensitivity of the regularization parameter is investigated on MNIST data and we set it to $0.05$ for all rest of the experiments. For MNIST dataset, the source and target domain data are augmented from MNIST training and test data respectively. For the fiber sensing dataset, the source and target domain data are collected in separate experimental rounds. Across all results tables, mean accuracy is outside the parenthesis, standard deviation is inside the parenthesis.

\subsection{Two-way MNIST} \label{subsec:mnist}
A variant of MNIST is created by rotation and translation operations. We rotate all MNIST digits by $[-30, -15, 0, 15, 30]$ degrees, and translate all MNIST by $[(-3, 0), (0, -3), (0,0), (0,3), (3,0)]$ pixels, leading to $5\times 5$ domains. 

For our proposed method, we use a simplified low-rank structure on the last two layers of LeNet. Specifically, LeNet has the structure of conv1-pool1-conv2-pool2-fc1-relu-fc2-relu-fc3-sigmoid. We impose the low-rank structure on both fc2 and fc3. 

 We create $11$ linear classifiers for each layer, denote as $s_1,\ldots, s_5, v_1,\ldots, v_5, u$. For task $(i,j)\in 5\times 5$, we use $s_i+v_j+u$ for prediction. This formulation is just a special case of the general formulation in \cref{eq:lowrankform}. Indeed, let $\alpha_1 = [
		s_1,  s_2, s_3 , s_4, s_5 , \bf 1 , \bf 1 ,  \bf 1 , \bf 1  , \bf 1 ]$, $\alpha_2 = [\mathbf{1} , \mathbf{1} ,  \mathbf{1} , \mathbf{1}  , \mathbf{1}, v_1,  v_2, v_3 , v_4, v_5 ]$,  and $\alpha_3 =  [
		u,  u, u, u, u, \bf 1 , \bf 1 ,  \bf 1 , \bf 1  , \bf 1 ]$, where each $w_i, v_i, u, \mathbf{1} \in\R^p$ and  $\alpha_k \in \R^{10}\times \R^p$.
 


Then for any task at $t=(i,j)\in 5\times 5$, its classifier $w_t$ this can be formally written as
$ 	w_t = \sum_{k=1}^3 \alpha_{k, i} \odot \alpha_{k, 5+j},
$ which is the same form as  \cref{eq:lowrankform}. This is done for both fc2 and fc3, and each layer has its own distinct set of weights.

Similar idea has been proposed for DG previously \cite{yang2016deep, li2017deeper}. These previous works do not assume a tensor structure on the tasks. Instead, they put a low-rank tensor structure on the classifiers themselves. This fundamentally distinguishes our settings from previous ones. As a result, during test phase they have to use the common classifier $u$ for different  target domains, but we can incorporate more prior knowledge by using $s_i+v_j+u$. 

During training time, we collect training data from $(i,i)$ entries for all $i\in[5]$, and leave  data in any other domains for test only. This is one of the designs requires the minimal number of source domains, and we can still successfully train all unknown classifiers $s$, $v$ and $u$. Due to the excess amount of learnable parameters, it is easy to overfit on our method. To reduce model complexity, we regularize all learnable classifiers to be close to their mean. That is, on each low-rank layer, let $\mu=\frac{1}{11}(\sum_iv_i+\sum_js_j+u)$, we add the following regularizer to the loss function $ \Omega_\lambda(s, v, u) = \frac{\lambda}{11}\p{\sum_i \|v_i-\mu\|^2+\sum_j\|s_j-\mu\|^2+\|u-\mu\|^2}$.

We run the experiments $10$ times and report the mean performances and standard deviations in \Cref{table:mnist}. Since this method uses the domain description information $(i,j)$ during testing, we refer to it as the \emph{domain-specific} model. In addition, we also report the performance of using just the common classifier $u$ in fc2 and fc3. This model uses no domain information during testing, and we call it the \emph{domain-agnostic} model. The baseline model is LeNet trained with data pooled from all source domains together. Notice that each $s_i+v_j+\nu$ corresponds to a unique task descriptor $q_{i,j}$ that serves as the coefficients for the combination of the linear layers, so for the ZSDA model, we further trained another descriptor network that outputs $\mathrm{ReLU}(Wq)$ as the new coefficients. This is done for both the last two layers. Our method almost uniformly dominates the baseline in every domain, sometimes by a large margin, and almost matches the performance of the ZSDA model, with less parameters. The domain-agnostic model achieves comparable performances to our proposed method. This shows that with the low-rank tensor structure, domain-agnostic models can also accommodate the domain shifts on this dataset.

Another interesting observation is that the performances of all models are better near the diagonal, and getting worse away from the diagonal. This may provide insight into how we should design experiments under a fixed budget of data collection. A conjecture is that the performances on unseen target domains relate to the Manhattan distance of these domains to the seen source domains. Further discussion is deferred to the appendix.

\begin{table}[t]
\caption{Test accuracy with $5$ observed domains on the diagonal. In each cell, from the $1$st to $4$th row: baseline, our domain-agnostic, and domain-specific models, both with the special low-rank formulation, and the general ZSDA model~\cite{yang2014unified}.}
\label{table:mnist}
\begin{center}
\vskip -0.1in
\resizebox{0.9\textwidth}{0.23\textheight}{
\begin{tabu}{c|[0.8pt] c|c|c|c|c}
\toprule
     & (-3, 0) &  (0, -3)& (0,0)& (0,3)& (3,0) \\ \tabucline[0.7pt]{-}
-30 &  &\makecell{ 0.958(0.007) \\ 0.950(0.007) \\0.965(0.004)\\\bf 0.967(0.002)} &\makecell{ 0.927(0.007) \\ 0.932(0.010) \\\bf 0.943(0.009) \\0.936(0.006)} &\makecell{ 0.735(0.017) \\ 0.769(0.024) \\\bf 0.775(0.024) \\0.769(0.022)} &\makecell{ 0.585(0.016) \\  0.659(0.037) \\0.646(0.038) \\\bf0.671(0.021)} \\ \hline
-15 & \makecell{ 0.975(0.003) \\ \bf 0.978(0.003) \\0.977(0.004) \\0.976(0.005)} & &\makecell{ 0.974(0.002) \\ 0.973(0.004) \\\bf 0.975(0.003) \\0.974(0.004)} &\makecell{ 0.908(0.005) \\ 0.907(0.010) \\ 0.911(0.010) \\ \bf 0.913(0.012)} &\makecell{ 0.797(0.009) \\ 0.846(0.018) \\0.839(0.015) \\\bf 0.852(0.012)} \\ \hline
0 & \makecell{ 0.925(0.012) \\ \bf 0.951(0.008) \\0.950(0.009) \\0.945(0.011)} &\makecell{ 0.969(0.004) \\ 0.966(0.005) \\\bf 0.971(0.005) \\\bf 0.971(0.004)} & &\makecell{ 0.973(0.002) \\ 0.971(0.004) \\ 0.976(0.002) \\\bf 0.977(0.001)} &\makecell{ 0.935(0.007) \\ 0.947(0.007) \\\bf 0.952(0.005) \\\bf 0.952(0.004)} \\ \hline
15 & \makecell{ 0.739(0.038) \\ \bf 0.810(0.029) \\0.801(0.027) \\0.804(0.019)} &\makecell{ 0.861(0.023) \\ 0.863(0.013) \\0.866(0.013) \\\bf 0.882(0.017)} &\makecell{ 0.974(0.003) \\ 0.971(0.007) \\\bf 0.978(0.003) \\\bf 0.978(0.001)} & &\makecell{ 0.975(0.003) \\ 0.975(0.002) \\0.977(0.002) \\\bf 0.978(0.002)} \\ \hline
30 & \makecell{ 0.494(0.039) \\ \bf 0.576(0.048) \\0.573(0.045) \\0.568(0.022)} &\makecell{ 0.649(0.027) \\ 0.664(0.018) \\0.681(0.024) \\\bf 0.715(0.018)} &\makecell{ 0.919(0.010) \\ 0.917(0.021) \\\bf 0.942(0.008) \\0.938(0.008)} &\makecell{ 0.955(0.004) \\ 0.930(0.012) \\\bf 0.956(0.006) \\0.947(0.016)} & \\
\bottomrule
\end{tabu}
}
\end{center}
\vskip -0.2in

\end{table}

\subsection{Four-way Fiber Sensing Dataset}

The distributed optic fiber sensing technique turns underground cable of tens of kilometers into a linear array of sensors, which could be used for traffic monitoring in smart city applications. In this paper, we aim to build a classifier for automatic vehicle counting and run-off-road event detection. 

The objective is to classify whether the sensed vibration signal is purely ambient noise, or it contains vibration caused by a vehicle either driving normally or crossing the rumbling stripes alongside the road. The sensing data takes the form of a 2D spatio-temporal array that can be viewed as an image - each pixel represents the vibration intensity at a particular time and location along the cable. Vehicles driving on the road,  running-off-road, and ambient noise all create different patterns, which makes convolutional neural networks a natural choice for a $3$-class classification model.  


In experimental design, we consider several influencing factors representing the practical challenges faced after field deployment. These factors include weather-ground conditions (sunny-dry, rainy-wet), shoulder type (grass, concrete), sensing distance ($0$km, $10$km, $15$km), and vehicle type (van, sedan, truck), which produce combined effects on the data domain. In order to evaluate our approach, we spent efforts on collecting data with all possible combinations of the levels in the aforementioned factors in a lab testbed, resulting to a multiway data indexed by a $2\times 3\times 3\times 2$ tensor. Under each condition, we conduct $20\sim 25$ rounds for the run-off-road events and normal driving. The ambient noise data is collected when no vehicles are near the cable. The classes are balanced in both training and test data. 


\begin{figure}[bpht]
\vskip -0.05in
\begin{center}
\tiny{$
\begin{array}{cc}
\hspace{-0.4cm}\includegraphics[height=0.35\textwidth, width=0.34\textwidth]{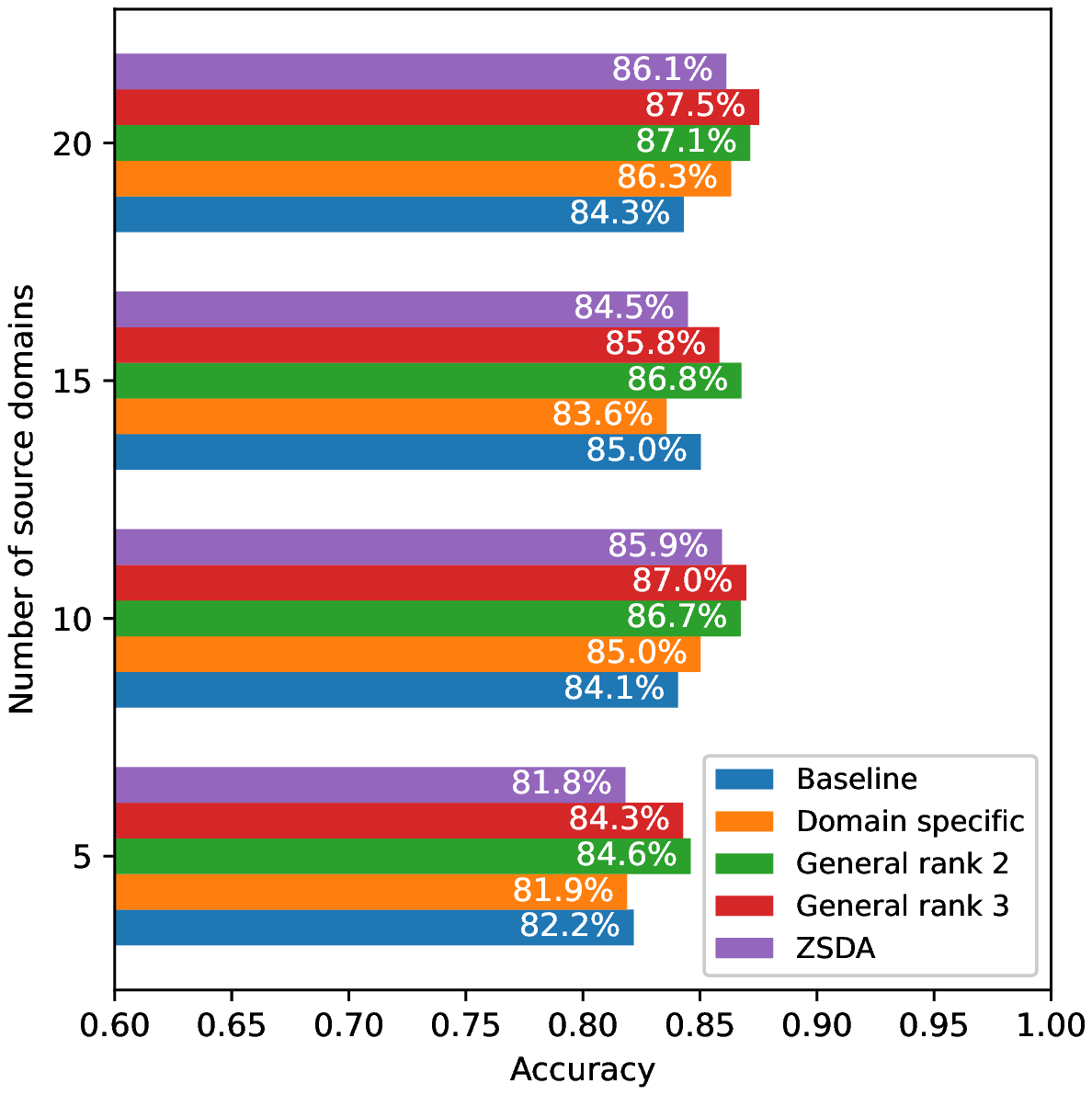} & \includegraphics[height=0.35\textwidth, width=0.65\textwidth]{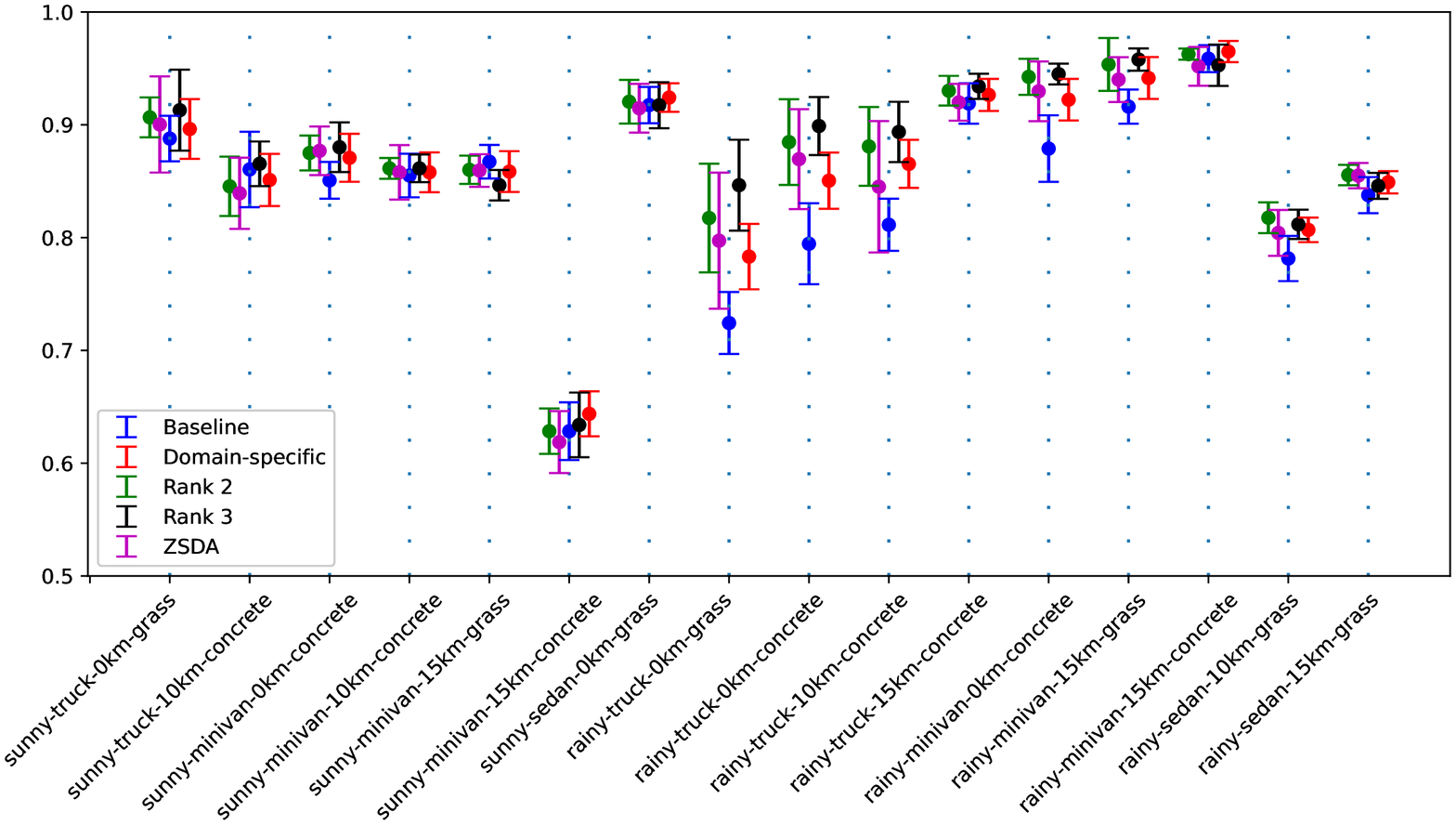} 
\\ (a) & (b)
\end{array}$}
\end{center}
\caption{\textbf{Left:} Average test accuracy vs. number of source domains. \textbf{Right:} Test accuracy, trained on $20$ random observed domains.}
\label{fig:fiber_sensing}
\end{figure}

We evaluate model performance against number of source domains ($5, 10, 15$, or $20$, each with $10$ simulation runs) and results are shown in Figure \ref{fig:fiber_sensing}(a). The performance increases with number of source domains as expected.  In particular, Figure \ref{fig:fiber_sensing}(b) summarizes the test performance for models trained with $10$ source domains and tested on the rest. In each run, the observed source domains are randomly sampled. Among the $10$ runs, "rainy-sedan-15km-grass" happens to be always selected as the source domain, thus the test result is not available for this domain. 

We add two additional low-rank models in the form of \cref{eq:endtoendlowrank} with $K=2$ and $3$ just in the last layer for comparison. Among the combination of factor levels, some scenarios are more challenging than others. For example, sedan is the lightest vehicle among the three, which generates much weaker vibration than a truck. Also the signal-to-noise ratio decays with increasing sensing distance. Results show that in most of the cases, our DG models achieve improved performance over the baseline, and the margins over baseline are particularly large in several challenging scenarios.




\section{Conclusion}
In this work, we propose a particular domain adaptation framework where $T$ out of $D=d^M$ total domains are observed during training. All $D$ domains are parameterized by a common latent representation and their domain-specific linear functionals, which form a $d^{\times M}$-dimensional low-rank tensor. This multilinear structure allows us to achieve an average excess risk of order $\tilde\gO\p{\p{\frac{TC(\gW)+C(\Phi)}{nT}}^{1/4}+p\p{\frac{KdM^2}{T}}^{1/2} }$. In addition to domain adaptation, our setting also sheds light on more efficient experiment designs and data augmentations. Algorithms developed under our framework are empirically verified on both benchmark and real-world datasets.

\section*{Acknowledgement}
SSD acknowledges support
from NEC. SH would like to thank Yuheng Chen, Ming-Fang Huang, and Yangmin Ding from NEC Labs America for their help with the fiber sensing data collection. 
\bibliographystyle{unsrtnat}
\bibliography{example_paper}

\newpage
\appendix

 \onecolumn
\section{Proofs in \Cref{sec:theory}}
\lineardiff*

\begin{proof}
	By a standard decomposition, we have with probability at least $1-2\delta$
	\begin{align*}
		&\frac{1}{T}\sum_{t\in[T]} \p{L(\hat w_t\circ\hat\phi)-L(w_t^*\circ\phi^*)}\\
		=&\frac{1}{T}\sum_{t\in[T]} \p{\E_{\gD_t}[\ell(\hat w_t\circ\hat\phi)]-\E_{\gD_t}[(w_t^*\circ\phi^*)]}\\
		\leq&\frac{1}{T}\sum_{t\in[T]} \E_{\gD_t}[\ell(\hat w_t\circ\hat\phi)]- \frac{1}{T}\sum_{t\in[T]} \sum_{i\in[n]}\ell((\hat w_t\circ\hat\phi)(x_{t,i}), y_i) \\
		& + \frac{1}{T}\sum_{t\in[T]} \sum_{i\in[n]}\ell((\hat w_t\circ\hat\phi)(x_{t,i}), y_i) - \frac{1}{T}\sum_{t\in[T]} \sum_{i\in[n]}\ell(( w_t^*\circ\phi^*)(x_{t,i}), y_i)\\
		& +\frac{1}{T}\sum_{t\in[T]} \sum_{i\in[n]}\ell(( w_t^*\circ\phi^*)(x_{t,i}), y_i) - \E_{\gD_t}[\ell(w_t^*\circ\phi^*)]\\
		\leq &2 \sup_{\substack{w_1,\ldots, w_T\in\gW,\\ \phi\in\Phi}}\abs{\frac{1}{T}\sum_{t\in[T]} \E_{\gD_t}[\ell( w_t\circ\phi)]- \frac{1}{T}\sum_{t\in[T]} \sum_{i\in[n]}\ell(( w_t\circ\phi)(x_{t,i}), y_i)}\\
		\leq&4\fR_{nT}(\ell(\gW^{\otimes T}\circ\Phi))+4B\sqrt{\frac{\log(1/\delta)}{nT}}\\
		\leq&8L\fR_{nT}(\gW^{\otimes T}\circ\Phi)+8B\sqrt{\frac{\log(1/\delta)}{nT}}
	\end{align*}
	Use the fact that $\wh \fR_n(\gH)\leq \sqrt{\pi/2} \wh\fG_n(\gH)$ for any function class $\gH$ and the decomposition of $\fR_{nT}(\gW^{\otimes T}\circ\Phi)$ introduced in \citet{tripuraneni2020theory}, we conclude that
	\begin{align*}
		&\frac{1}{T}\sum_{t=1}^T\p{L(\hat w_t\circ \hat\phi) - L(w_t^*\circ\phi^*)}\leq \underbrace{4096 L\left[\frac{WD_{\mathcal{X}}}{(n T)^{2}}+\log (n T) \cdot\left[W \cdot \mathfrak{G}_{n T}(\Phi)+\overline{\mathfrak{G}}_{n}(\mathcal{W})\right]\right]
		+8 B \sqrt{\frac{\log (2 / \delta)}{n T}} }_\text{\textcircled{1}}\\
	\end{align*}
	Meanwhile, by our assumption of uniform absolute representation difference and $(\nu,\epsilon)$-task diversity:
	\begin{align*}
		\frac{1}{T}\sum_{t=1}^T\p{L(\hat w_t\circ \phi^*) - L(\hat w_t\circ\hat\phi)}\leq \frac{1}{T}\sum_{t=1}^T\abs{L(\hat w_t\circ \phi^*) - L(\hat w_t\circ\hat\phi)}\leq d_\gW(\hat \phi; \phi^*) \leq \bar d_{\gW, w^*}(\hat \phi, \phi^*)/\nu+\epsilon,
	\end{align*}
	which gives
	\begin{align*}
		&\frac{1}{T}\sum_{t=1}^T\p{L(\hat w_t\circ \phi^*) - L(w_t^*\circ\phi^*)} = \frac{1}{T}\sum_{t=1}^T\p{L(\hat w_t\circ \phi^*) - L(\hat w_t\circ\hat\phi)} + \frac{1}{T}\sum_{t=1}^T\p{L(\hat w_t\circ \hat\phi) - L( w_t^*\circ\phi^*)}\\
		& \leq  \textcircled{1} + \textcircled{1}/\nu+\epsilon,
	\end{align*}
	where we use $d_{\gW, w^*}(\hat \phi, \phi^*)\leq \textcircled{1}$ from \Cref{thm:repdiv}.
	
%
	By strong convexity, we conclude for any $T\in[D]$, $\nu\leq 1$:
	\begin{align*}
		&\frac{1}{T}\sum_{t\in[T]}\|\hat w_t-w^*_t\|_2 \leq \sqrt{\frac{2}{\lambda}} \p{\frac{1}{T}\sum_{t\in[T]}\p{L(\hat w_t\circ \phi^*) - L(w_t^*\circ\phi^*)}}^{1/2}\\
		&\leq \sqrt{\frac{4}{\lambda\nu}}\Bigg(4096 L\Big[\frac{WD_{\mathcal{X}}}{(n T)^{2}}+\log (n T) \cdot\left[W \cdot \mathfrak{G}_{n T}(\Phi)+\overline{\mathfrak{G}}_{n}(\mathcal{W})\right]\Big] +8 B \sqrt{\frac{\log (2 / \delta)}{n T}}+\epsilon\Bigg)^{1/2}.
	\end{align*}

\end{proof}

Now we proceed to show the results related to tensor completion. The main idea is to treat a tensor as a polynomial, and count its connect components. This number restricts the complexity of the tensor.

\begin{corollary}\label{cor:numconfig}
	The number of $\{\pm 1, 0\}$ sign configurations of $r$ polynomials of degree at most $d$, over $q$ variables, is at most $(8edr/q)^q$ for $r>q>2$.
\end{corollary}

\pseudodim*
\begin{proof}
	Let $f_T(d, M, K)=\abs{\{\operatorname{sign}(X-T) \in \{\pm 1,  0\}^{d^{\times M}}: X\in \gX_K\}}$. It suffices to show that $f_T(d, M, K)\leq (8ed)^{KdM^2}$. A rank $K$ $d^{\times M}$ tensor can be decomposed as:
	\[
		X_{t}= \sum_{k=1}^K\prod_{m=1}^M U_{k, t_m}
	\]
	for $t\in [d]^{\times M}$. Then one can treat $X-T$ as $d^M$ polynomials of degree at most $M$ over the following entries:
	\[
		(X-T)_t =  \sum_{k=1}^K\prod_{m=1}^M U_{k, t_m} - T_t.
	\]
	$T$ is a fixed arbitrary tensor, so there are in total $KdM$ variables. Applying \Cref{cor:numconfig} yields the desired result.
\end{proof}

\tensorcompletion*

\begin{proof}
	The following equation follows from the uniform convergence bound using covering number and bounding covering number using pseudo-dimension. See \citet{srebro2005rank} for detail. 
	
	\begin{align}\label{eq:tensorcomplete}
	\begin{split}
			\frac{1}{D}\sum_{t\in[D]} \abs{\wh\gT_{t,j}-\tilde \gT_{t,j}}\leq \frac{1}{T}\sum_{t\in[T]} \abs{\wh\gT_{t,j}-\tilde\gT_{t,j}}+\sqrt{\frac{KdM^2\log\p{8ed}-\log\delta}{T}}.
	\end{split}
	\end{align}
	
	Then by triangle inequality and equivalence of norms in finite-dimensional spaces
	\begin{align*}
			\frac{1}{D}\sum_{t\in[D]} \norm{\wh\gT_{t,\cdot}- \gT^*_{t,\cdot}}&\leq \frac{1}{D}\sum_{t\in[D]} \norm{\wh\gT_{t,\cdot}- \tilde\gT_{t,\cdot}}+\frac{1}{D}\sum_{t\in[D]} \norm{\tilde\gT_{t,\cdot}- \gT^*_{t,\cdot}}\\
			&\leq \frac{1}{T}\sum_{t\in[T]} \sum_{j\in[p]}\abs{\wh\gT_{t,j}-\tilde\gT_{t,j}}+p\sqrt{\frac{KdM^2\log\p{8ed}+\log(p/\delta)}{T}}+\cO(n^{-1/4}),
	\end{align*}
	where the first step is from the fact that $\|\cdot\|_2\leq \|\cdot\|_1$ and union bounding \cref{eq:tensorcomplete} over $p$ events. The $\gO(n^{-1/4})$ term follows from \Cref{thm:lineardiff}. Again, as mentioned in \Cref{remark:tensorcomplete}, one should imagine that every entry in $\tilde \gT$ has been hypothetically learned to $\epsilon(n)$ close to $\gT^*$, but only $T$ entries have been revealed to the learner during the tensor completion step. Keep in mind that $D-T$ entries of $\tilde\gT$ only exist hypothetically and are never computed.
\end{proof}

\mainthm*
\begin{proof}
	\begin{align*}
		&\frac{1}{D}\sum_{t=1}^D\mathop E_{(x,y)\sim \gD_t}[\ell(\hat w_t\circ\hat\phi(x), y)-\ell( w^*_t\circ\phi^*(x), y)]\\
		= & \underbrace{ \frac{1}{D}\sum_{t=1}^D\mathop E_{(x,y)\sim \gD_t}[\ell(\hat w_t\circ\hat\phi(x), y)-  \ell( \hat w_t\circ\phi^*(x), y)] }_{A}+\underbrace{\frac{1}{D}\sum_{t=1}^D\E_{\gD_t}[\ell( \hat w_t\circ\phi^*(x), y)- \ell( w^*_t\circ\phi^*(x), y)]}_\text{B}\\
	\end{align*}
	
	By our assumption of uniform absolute representation difference and $(\nu, \epsilon)$-diverse, we can upper bound $A$ by $ \bar d_{\gW,w^*}(\hat\phi, \phi^*)/\nu+\epsilon$. The second term can be bounded by lipschitzness,
	\begin{align*}
		B&\leq \frac{1}{D}\sum_{t\in[D]}\E_{\gD_t}\left[ L\abs{\hat w_t\circ\phi^*(x)-w^*_t\circ\phi^*(x)} \right]\\
				&\leq \frac{1}{D}\sum_{t\in[D]} \E_{\gD_t}\left[ L\|\hat w_t-w^*_t\| \|\phi^*(x)\| \right]\\
				&\leq \frac{LD_\gX}{D}\sum_{t\in[D]} \|\hat w_t-w^*_t\|.
	\end{align*}
	
	Plug in our approximation to $d_{\gW,w^*}(\hat\phi, \phi^*)$ in \Cref{thm:repdiv}, $ \frac{1}{D}\sum_{t\in[D]} \|\hat w_t-w^*_t\|$ in \Cref{thm:mainthm}, and union bound, we conclude the theorem.
\end{proof}

\section{More Discussion to \Cref{sec:experiment}}

We mention in \Cref{subsec:mnist} that the test accuracy on unseen domain might relate to the Manhattan distance between the seen and unseen domains. Here is another experiment in the same setting as in \cref{subsec:mnist}, but the chosen observed entries are [(0,0), (1,3), (2,2), (3,1), (4,4)].

\begin{table}
\caption{Mean test accuracy for our method (both domain-specific and domain-agnostic) and baseline. Same settings as in \Cref{table:mnist} but difference domains are observed.}
\label{table:mnist_design2}
\vskip 0.1in
\centering
\resizebox{\textwidth}{0.165\textheight}{

\begin{tabu}{c|[0.8pt] c|c|c|c|c}
\toprule
 \diagbox{\textsc{Rotation}}{\textsc{Translation}}      & (-3, 0) &  (0, -3)& (0,0)& (0,3)& (3,0) \\ \tabucline[0.7pt]{-}
-30 &  & \makecell{ 0.740(0.021) \\ 0.767(0.024) \\0.796(0.022)} &\makecell{ 0.932(0.009) \\ 0.940(0.008) \\0.947(0.008)} &\makecell{ 0.960(0.005) \\ 0.958(0.007) \\0.965(0.003)} &\makecell{ 0.604(0.025) \\ 0.654(0.044) \\0.644(0.043)} \\ \hline
-15 & \makecell{ 0.971(0.003) \\ 0.977(0.003) \\0.977(0.003)} &\makecell{ 0.906(0.012) \\ 0.915(0.007) \\0.914(0.007)} &\makecell{ 0.976(0.002) \\ 0.978(0.002) \\0.980(0.002)} & & \makecell{ 0.829(0.016) \\ 0.860(0.021) \\0.866(0.019)} \\ \hline
0 & \makecell{ 0.919(0.011) \\ 0.953(0.006) \\0.950(0.007)} &\makecell{ 0.973(0.004) \\ 0.972(0.003) \\0.975(0.003)} & & \makecell{ 0.969(0.003) \\ 0.961(0.007) \\0.969(0.005)} &\makecell{ 0.948(0.006) \\ 0.952(0.006) \\0.955(0.006)} \\ \hline
15 & \makecell{ 0.756(0.024) \\ 0.844(0.020) \\0.813(0.019)} & & \makecell{ 0.966(0.005) \\ 0.969(0.005) \\0.967(0.008)} &\makecell{ 0.861(0.016) \\ 0.844(0.024) \\0.853(0.025)} &\makecell{ 0.978(0.002) \\ 0.976(0.002) \\0.977(0.003)} \\ \hline
30 & \makecell{ 0.554(0.021) \\ 0.657(0.029) \\0.656(0.028)} &\makecell{ 0.958(0.006) \\ 0.956(0.005) \\0.966(0.005)} &\makecell{ 0.903(0.021) \\ 0.914(0.018) \\0.940(0.011)} &\makecell{ 0.656(0.025) \\ 0.632(0.027) \\0.691(0.030)} & \\
\bottomrule
\end{tabu}
}
\vskip -0.1in
\end{table}

In both \Cref{table:mnist_design2} and \Cref{table:mnist}, every factor level has been observed exactly once. If the factors are categorical nominal, then permuting rows $2$ and row $4$ in \Cref{table:mnist_design2} leads to the same balanced design with \Cref{table:mnist}, and the performance shall be similar. However, the factors (rotation, translation) considered here are not nominal, for example, there is a ordinal relation between rotation $-30$ and rotation $-15$. Hence, the above-mentioned permutation seems prohibitive. Consequently, it makes sense to talk about Manhattan distances between observed and unobserved entries.

In \Cref{table:mnist}, cell (4,0) is 4 units away from its closest observed cell in Manhattan distance (in the following we omit to mention the metric is Manhattan distance), but in \Cref{table:mnist_design2}, (4,0) is 2 units away from the closest observed entry. One can see the accuracy in that domain gets much better. However, the closest distance to the observed entry is not the only factor here. For example, (0,1) in \Cref{table:mnist} outperforms (0,1) in \Cref{table:mnist_design2} a lot. Therefore, the average distance to the observed entries may also be a contributing factor.

It is an open question how to select the best subset of tensor entries to be observed, such that the overall performance in the unseen target domains can be optimized. This question may relate to the area of factorial experiment design, where many factors are involved and some have confounded interactions. One subject is to design minimal sets of experiments such that the effects of all factors can be studied. Our theory deals particularly with the case when the  factor combinations are chosen uniformly at random, without assuming any particular tensor structure. With more prior knowledge, more efficient sampling algorithms can be designed. For instance, \citet{mckay2000comparison} discussed how to use Latin hypercube sampling to achieve a smaller sample complexity. There are also some works that connect the dots between error coding theory and factorial experiment design \cite{ben2001application}. There could be information theoretic understanding to the relation between our setting of experiment design and Manhattan distance as well. This can be an interesting research direction and beyond the scope of our paper.

\paragraph{Hyperparameter Sensitivity} Model selection in domain adaptation is tricky in general, since no data from the target domains is seen. One benefit of using the special low-rank formulation is that it has fewer number of tuning parameters than general low-rank formulations. In addition to tuning $\lambda$ in the regularizer $\Omega_\lambda$, the form in \cref{eq:endtoendlowrank} also requires to choose rank $K$.  We evaluate the sensitivity of $\lambda$ on both the source and target domains. Observing data from $(i,i)$ domains on the diagonal, we train on $5000$ training data using $\lambda\in[0.005,0.01,0.03,0.05,0.1,0.5,1]$, and test on $1000$ data from both $(i,i)$ source domains for $i\in[5]$ and $(i,j)$ task domains for $i\neq j \in[5]$.  Results show that the test performances on both source domains and target domains are insensitive to $\lambda$. The mean performances and standard deviations are reported in \Cref{table:model_selection}.

\begin{table}[htbp]
\caption{Hyperparameter sensitivity.}
\label{table:model_selection}
\centering
\vskip 0.1in
\begin{tabu}{ccc}
\toprule
$\lambda$ & \textsl{Source Domain} & \textsl{Target Domain}\\ \midrule
0.005 & 0.963 (0.003) & 0.822 (0.005) \\ \hline
0.01 & 0.964 (0.003) & 0.821 (0.010)\\ \hline
0.03 & 0.962 (0.003)& 0.824 (0.006)\\ \hline
0.05 & 0.963 (0.003)& 0.818 (0.007)\\ \hline
0.1 & 0.959 (0.003)& 0.823 (0.008)\\ \hline
0.5 & 0.961 (0.005)& 0.810 (0.010)\\ \hline
1 & 0.962 (0.004)& 0.811 (0.001)\\ 
\bottomrule
\end{tabu}
\vskip -0.1in
\end{table}

\end{document}